\documentclass[conference]{IEEEtran}
\usepackage[utf8]{luainputenc}

\usepackage{cite}
\usepackage{amsmath}
\usepackage{amsthm}
\usepackage{amsfonts}
\usepackage{mathrsfs}
\usepackage{graphicx}
\usepackage{microtype}
\usepackage{textcomp}
\usepackage{xcolor}
\usepackage{booktabs}
\usepackage{mathtools}
\usepackage{algorithm}
\usepackage{algorithmic}
\usepackage{textcomp}
\usepackage{xurl}
\usepackage{braket}
\usepackage{balance}
\usepackage[disableredefinitions,basic]{complexity}
\RequirePackage[normalem]{ulem}
\RequirePackage{color}\definecolor{RED}{rgb}{1,0,0}\definecolor{BLUE}{rgb}{0,0,1}

\usepackage{soul,xcolor}

\DeclareMathOperator*{\argmax}{arg\,max}

\graphicspath{{./figures/}}

\newtheorem{theorem}{Theorem}
\newtheorem{lemma}{Lemma}

\newcommand{\LINEIFRETURN}[2]{%
    \STATE\algorithmicif\ {#1}\ \algorithmicthen\ \algorithmicreturn\ {#2} \algorithmicend\ \algorithmicif%
}

\begin{document}
\bstctlcite{IEEEexample:BSTcontrol}

\title{Counterfactual Explanations for Machine Learning on Multivariate Time Series Data}

\author{
   \IEEEauthorblockN{Emre Ates\IEEEauthorrefmark{1},
     Burak Aksar\IEEEauthorrefmark{1},
     Vitus J. Leung\IEEEauthorrefmark{2}, and
     Ayse K. Coskun\IEEEauthorrefmark{1}}
   \IEEEauthorblockA{
     \IEEEauthorrefmark{1}\textit{Dept. of Electrical and Computer Eng.} \\
     \textit{Boston University}\\
     Boston, MA, USA \\
     Email: \{ates,baksar,acoskun\}@bu.edu}
   \IEEEauthorblockA{
     \IEEEauthorrefmark{2}\textit{Sandia National Laboratories}\\
     Albuquerque, NM, USA\\
     Email: vjleung@sandia.gov}
}

\maketitle

\begin{abstract}
  Applying machine learning (ML) on multivariate time series data has growing popularity in many application domains, including in computer system management. For example, recent high performance computing (HPC) research proposes a variety of
  ML frameworks that use system telemetry data in the form of
  multivariate time series so as to detect performance variations,
  perform intelligent scheduling or node allocation, and improve system security.
  Common barriers for adoption for these ML frameworks include the
  lack of user trust and the difficulty of debugging. These barriers need to be
  overcome to enable the widespread adoption of ML frameworks in production
  systems.
  To address this challenge, this paper proposes a novel explainability technique for
  providing counterfactual explanations for supervised ML
  frameworks that use multivariate time series data. Proposed method outperforms state-of-the-art explainability methods on several different ML frameworks and
  data sets in metrics such as faithfulness and robustness. The paper also demonstrates how the proposed method
  can be used to debug ML frameworks and gain a better understanding
  of HPC system telemetry data.
\end{abstract}


\section{Introduction}


Multivariate time
series data analytics have been gaining popularity due to the recent advancements in
internet of things technologies and omnipresence of real-time sensors~\cite{Assaf:2019}. Health care,
astronomy, sustainable energy, and geoscience are some 
domains where researchers utilize multivariate time series along with machine learning based analytics to solve problems such as seismic activity forecasting, hospitalization rate prediction, and many others~\cite{Che:2018}. Large-scale computing system management has also been increasingly leveraging time series analytics for
improving performance, efficiency, or security.
For example, high performance computing (HPC)
systems produce terabytes of instrumentation data per day in the form of logs,
metrics, and traces, and HPC monitoring frameworks organize system-wide resource
utilization metrics as multivariate time series. Thousands of metrics can be
collected per node and each metric---representing different resource statistics
such as network packet counts, CPU utilization, or memory statistics---is
sampled on intervals of seconds to
minutes~\cite{Agelastos:2014,Bartolini:2018,Ates:2019}. Analyzing this data is
invaluable for management and
debugging~\cite{Bodik:2010,Ramin:2018,Agelastos:2014}, but extensive manual
analysis of these big data sets is not feasible.

Researchers have recently started using machine learning (ML)
to help analyze HPC system telemetry data and gain valuable insights.
ML methods
can process large amounts of data and, in addition, frameworks using ML methods
benefit from the flexibility of the models that generalize to different
systems and potentially previously unseen cases. ML frameworks
have been shown to diagnose performance
variations~\cite{Tuncer:2017,Tuncer:2019,Borghesi:2019,Klinkenberg:2017},
improve scheduling~\cite{Yang:2011,Xiong:2018} or improve system security by
detecting unwanted or illegal
applications on HPC systems~\cite{Ates:2018} using multivariate time series data.

While many advantages of ML are well-studied, there are also
common drawbacks that ML frameworks need to address
before they can be widely used in production. These frameworks commonly have
a taciturn nature, e.g., reporting only the final diagnosis when analyzing performance problems in HPC systems such as
``{\tt network contention on router-123},'' without providing reasoning relating
to the underlying data. Furthermore, the ML models within these
frameworks are black boxes which may
perform multiple data transformations before arriving at a classification, and
thus are often challenging to understand. The black-box nature of these frameworks
causes a multitude of drawbacks, including making it challenging to debug
mispredictions, degrading user trust, and reducing the overall usefulness of the
systems.

To address the broad ML {\em explainability} problem, a number of methods
that explain black-box classifiers have been proposed by researchers~\cite{Arya:2019}.
These methods can be
divided into {\em local} and {\em global} explanations, based on whether they
explain a single prediction or the complete classifier. Local explanations can
also be divided into {\em sample-based} explanations that provide different
samples as explanations and {\em feature-based} explanations that indicate the
features that impact the decision the most. However, most of existing
explainability methods are not designed for multivariate time series data, and
they fail to generate sufficiently simple explanations when tasked with complex multivariate time series data, such as in explaining
ML frameworks for analyzing HPC systems.

\begin{figure}[t]
  \centering
  \includegraphics[width=\columnwidth]{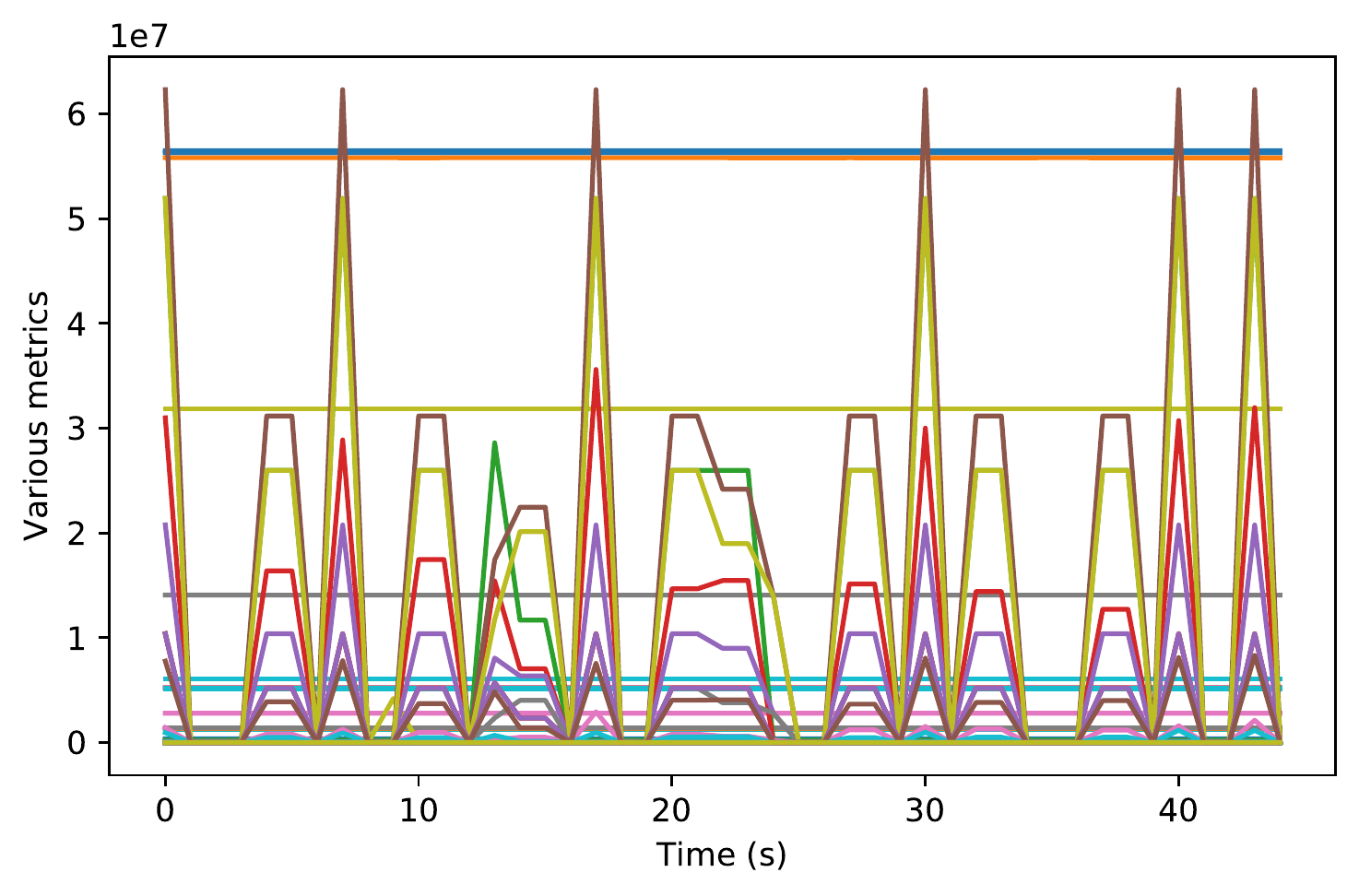}
  \vspace{-0.3in}
  \caption{A 45-second time window sample from a single node of a multi-node
    application execution on Voltrino, a Cray XC30m supercomputer with 56 nodes. Existing
    explainability techniques do not address the problem of the inherent
    complexity of HPC time series data. Sample-based techniques assume that
    users can interpret the difference between two samples and
    feature-based methods assume users understand the meaning of every feature
    (thus, they do not limit the number of features in explanations).
  }\label{fig:timeseries}
  \vspace{-0.2in}
\end{figure}

Why do existing explainability methods fail to provide satisfactory explanations
for high-dimensional multivariate time series data? One differentiating factor is the complexity of the
data. Figure~\ref{fig:timeseries} shows a {\em sample} from HPC domain, where 199 metrics are collected
through the system telemetry during an application run.
Existing sample-based methods provide samples from the training set
or synthetically generate samples~\cite{Koh:2017,Cem:2018}. These methods are designed
with the assumption that one sample is self-explanatory and users can visually distinguish
between two samples; however, providing another HPC time series sample with
hundreds of metrics similar to the one shown in Fig.~\ref{fig:timeseries} is most often not an
adequate explanation. On the other hand, existing feature-based
methods~\cite{Ribeiro:2016,Lundberg:2017} provide a set of features and expect
the users to know the meaning of each feature, as well as normal and abnormal
values for them, which is often not possible in many domains, including HPC.

In this paper, we introduce a novel explainability method for time series data
that provides {\em counterfactual} explanations for individual predictions. The
counterfactual explanations consist of hypothetical samples that are as similar
as possible to the sample that
is explained, while having a different classification label, i.e., ``if
these particular metrics were different in the given sample, the classification label
would have been different.'' The format of the counterfactual explanation is identical
to the format of the training data, but only pinpointing the metrics with the replaced  time series.
The counterfactual explanations
are generated by selecting a small number of time series
from the training set and substituting them in the sample under investigation to obtain
different classification results. In this way, end users
can understand the expected behavior by examining a limited number of substituted metrics.
These explanations can be then used to debug misclassifications, understand how the
classifier makes decisions, provide adaptive dashboards that highlight important
metrics from ongoing application runs, and extract knowledge on the nature of normal or
anomalous behavior of a system.

Our specific contributions are as follows:
\begin{itemize}
\item Demonstration of existing general-purpose explainability methods and how
  they are inadequate to explain ML frameworks that work with multivariate time series data, particularly focusing on ML frameworks for HPC system analytics~(\S~\ref{sec:eval:qual}),
\item design of a formal problem statement for multivariate time series explainability
  and a proof of \NP-hardness of the problem~(\S~\ref{sec:problem}, Appendix),
\item design of a heuristic algorithm for the time series explainability
  problem~(\S~\ref{sec:method}),
\item demonstration of the application of the proposed explainability method to several
  ML frameworks that work with multivariate data sets using four different data sets~(\S~\ref{sec:setup},~\ref{sec:eval}),
\item comparison of our method with state-of-the-art explainability methods
  using a set of novel and standard metrics. Our method generates comprehensible
  explanations for the different time series data sets we use (i.e., several HPC telemetry data sets and a motion classification data set)  and performs better than baselines in
  terms of {\em faithfulness} and {\em robustness}~(\S~\ref{sec:eval}).
\end{itemize}


\section{Related Work}

Explainability has been the topic of much research in the last few years. We use
the helpful classification of Arya et al. when discussing the existing
literature~\cite{Arya:2019}. We focus on explainable models and local sample- and
feature-based explanations for black-box models.

\textbf{Explainable models}
learn a model that is inherently understandable by untrained
operators. They include simple models like logistic regression and decision trees
and newer models like CORELS, which learns minimal rule lists that are
easy to understand~\cite{Angelino:2017}. Our experiences with using CORELS with
HPC time series data in Sec.~\ref{sec:eval:explainable} show that
CORELS fails to learn a usable model, and decision trees using HPC data become
too complex to be understood by human operators.

\textbf{Sample-based explanations}
provide samples that explain decisions by giving supportive or counter examples,
or prototypes of certain classes.
Koh and Liang use influence functions to find training set samples that
are most impactful for a specific decision~\cite{Koh:2017}. Contrastive
explanations method (CEM) provides a synthetic
sample with a different classification result that is as
similar to the input sample as possible, while using autoencoders to ensure the
generated sample is realistic~\cite{Cem:2018}. Sample-based methods do not
address the inherent complexity of multivariate time series data, since they focus on tabular data
or images where the test sample and the explanation sample are easy to compare
manually. Data with thousands of time series per sample is
challenging for users to compare and contrast without additional computing.

\textbf{Feature-based explanations}
highlight certain features that are impactful during
classification. Global feature-based explanations for some
classifiers such as random forests and logistic regression use the learned
weights of the classifiers~\cite{Palczewska:2014}. Local feature-based
explanations include LIME, which fits a linear model to classifier decisions for
samples in the neighborhood of the sample to be explained~\cite{Ribeiro:2016}.
SHAP derives additive Shapley values for each feature to represent the impact
of that feature~\cite{Lundberg:2017}. These feature-based models do not support
using time series directly; however, many ML frameworks use feature transformations
and feature-based explanations can explain these models' classifications in terms of
the features they use. In this case, bridging the gap between feature-based
explanations and the complete framework is left to the user. This task can
become unfeasible for untrained users since complex features can be used
such as kurtosis or B-splines~\cite{Endrei:2018}.

\textbf{Time series specific explanations}
have also been foci of research. Schegel et al. evaluate different general
purpose explainability methods on time series~\cite{Schegel:2019}. Gee et al.
propose a method to learn
prototypes from time series~\cite{Gee:2019}. All of these methods operate
on univariate time series and do not address the problem of
explaining multivariate time series. Assaf and Roy propose a method
to extract visual saliency maps for multivariate time series~\cite{Assaf:2019};
however, their method is specific to deep neural networks. Furthermore, saliency
maps lose their simplicity when scaled to hundreds or thousands of time series.

\textbf{Counterfactual explanations}
have been used to explain ML models. Wachter et al. are among the first to use
the term ``counterfactual'' for ML explanations~\cite{Wachter:2017}.
Counterfactual explanations have also been
used in the domains of image~\cite{Goyal:2019}, document~\cite{Martens:2014} and
univariate time series~\cite{Karlsson:2018} classification. DiCE is an open
source counterfactual explanation method for black-box
classifiers~\cite{Mothilal:2020}. To the best of our knowledge, there are no
existing methods to generate counterfactual explanations for high-dimensional
multivariate time series.





\section{Counterfactual Time Series Explanation Problem}\label{sec:problem}
Our goal is to provide counterfactual explanations for ML methods
that operate on time series data. We define the {\em counterfactual time
series explanation problem} as follows. Given a black-box ML
framework that takes
multivariate time series as input and returns class probabilities, the
explanations show which time series need to be modified, and how, to change the
classification result of a sample in the desired way, e.g., ``if {\tt MemFree::meminfo} (a feature in an HPC performance analytics framework) was
not decreasing over time, this run would not be classified as a memory leak (an anomaly affecting performance of an HPC system).'' For a given
sample and a class of interest, our counterfactual explanation finds a small
number of substitutions to the sample from a {\em distractor}\footnote{
The distractor is a sample chosen from the training set that our methods
``distracts'' the classifier with, resulting in a new classification result.}
chosen from the
training set that belongs to the class of interest, such that the resulting sample is
predicted as the class of interest. We assume a black box model for the classifier,
thus having no access to the internal weights or gradients. We next define
our problem formally.

\subsection{Problem Statement}

In this paper, we represent multivariate time series classification models using $f(x) = y :
\mathbb{R}^{m \times t} \rightarrow \mathbb{R}^k$ where the model $f$ takes $m$
time series of length $t$ and returns the probability for $k$ classes. We use
the shorthand $f_c(x)$ as the probability for class $c \in [1, k]$. Our goal is
to find the {\em optimum counterfactual explanation} for a given test sample
$x_{test}$ and class of interest $c$. We define an optimum counterfactual
explanation as a modified sample $x'$ that is constructed using $x_{test}$ such
that $f_c(x')$ is maximized. The class of modifications that we consider to
construct $x'$ are substitutions of entire time series from a distractor sample
$x_{dist}$, chosen from the training set, to $x_{test}$. Our second objective is
to minimize the number of substitutions made to $x_{test}$ in order to obtain
$x'$.

The optimum counterfactual explanation can be constructed by finding $x_{dist}$
among the training set and $A$ which minimizes
\begin{equation}\label{eqn:loss}
L \left( f, c, A, x' \right) = \left( 1 - f_c(x') \right)^2 + \lambda ||A||_1,
\end{equation}
where
\begin{equation}\label{eqn:xprime}
  x' = (I_m - A) x_{test} + A x_{dist},
\end{equation}
$\lambda$ is a tuning parameter, $I_m$ is the $m \times m$ identity matrix, and
$A$ is a binary diagonal matrix where $A_{j,j} = 1$ if metric
$j$ of $x_{test}$ is going to be swapped with that of $x_{dist}$, 0 otherwise.

We prove in the Appendix that the problem of finding a counterfactual
explanation, $x_{dist}$ and $A$, that maximize $f_c(x')$ is \NP-hard. Because of
this, it is unlikely that a polynomial-time solution for our explainability
problem exists; therefore, we focus on designing approximation algorithms and
heuristics that generate acceptable explanations in a practical duration.

\subsection{Rationale for Chosen Explanation}
Explainability techniques targeting multivariate time series frameworks need to consider several
properties that general-purpose explainability techniques do not
consider. In most domains, such as HPC, time series data is more complex than traditional machine learning
data sets by several aspects. A single sample is most often not explainable because of the
volume of data contained, as shown in Fig.~\ref{fig:timeseries}. Therefore,
sample-based methods often fail to provide comprehensible explanations. Furthermore,
each metric in the time series
requires research to understand. For example, in an HPC telemetry system, 
performance counters with the same
name can have different meanings based on the underlying CPU model. Furthermore, the values
of metrics may not be meaningful without comparison points. In order to address
these challenges, our explainability approach is {\em simultaneously} sample- and
feature-based. We provide a counterfactual sample from the training set, and
indicate which of the many time series in the sample need to be modified to have
a different classification result. This results in an explanation that is easy
to understand by human operators, since it requires interpreting only a minimal
number of metrics. To help users interpret metric values, we provide
examples of the same metric for both the distractor and $x_{test}$.

Many existing explainability techniques rely on synthetic data generation either as
part of their method or as the end result~\cite{Ribeiro:2016,Cem:2018,Lundberg:2017,Wachter:2017}.
These explanations assume that the synthetic modifications lead to meaningful
and feasible time series. Generating synthetic time series data is challenging for a domain like HPC performance analytics because
many of the time series collected through the system telemetry 
represent resource utilization values that
have constraints; e.g., the rate of change of certain performance counters and the
maximum/minimum values of these counters are bounded by physical constraints of the CPUs, servers, and other components in the HPC system.

In our method, we choose $x_{dist}$ from the training set as part of the
explanation, in contrast to providing synthetic samples, which guarantees that the time
series in the explanation are feasible and
realistic, because they were collected from the same system. Choosing a distractor
$x_{dist}$ from the training set also enables administrators to inspect the logs
and other information besides time series that belong to the sample.

We keep the number of distractors to 1, instead of substituting individual time
series from various distractors, in order to guarantee a possible solution. As
long as $\argmax_{j \in [1, k]} f_j(x_{dist}) = c$, a solution with $||A||_1
\leq m$ exists. Furthermore, in cases where administrators need to inspect logs
or other related data, keeping the distractor count small helps improve the
usability of our method.


\section{Our Method: Counterfactual Explanations}\label{sec:method}
What are some algorithms that can be used to obtain counterfactual
explanations? We present a greedy search algorithm that generates
counterfactual explanations for a black-box classifier and a faster optimization of
this algorithm.

As we described in Sec.~\ref{sec:problem}, our goal is to find
counterfactual explanations for a given test sample $x_{test}$. Recall that a
counterfactual explanation is a minimal modification to $x_{test}$ such that the
probability of being part of the class of interest is maximized. Our method
aims to find the minimal number of time series substitutions from the chosen
distractor $x_{dist}$ instance that will flip the prediction.

We relax the loss function $L$ in Eqn.~\eqref{eqn:loss}, using
\begin{equation}\label{eqn:modloss}
  L \left( f, c, A, x' \right) = \left( (\tau - f_c(x'))^+ \right)^2 + \lambda (||A||_1 - \delta)^+,
\end{equation}
where $x'$ is defined in Eqn.~\eqref{eqn:xprime}, $\tau$ is the target probability
for the classifier, $\delta$ is the desired number of features in an explanation
and $x^+ = \max(0, x)$, which is the rectified linear unit (ReLU). ReLU is used
to avoid penalizing explanations shorter than $\delta$. Running optimization
algorithms until $f_c(x')$ becomes 1 is usually not feasible and the resulting
explanations do not significantly change; thus, we empirically set $\tau = 0.95$.
We set $\delta = 3$, as it is shown to be a suitable number of features in an
explanation~\cite{Miller:2018}.

Our explainability method operates by choosing multiple distractor candidates
and, then, finding the best $A$ for each distractor. Among the different $A$
matrices, we choose the matrix with the smallest loss value. We present our
method for choosing distractors, and two different algorithms for choosing
matrix $A$ for a given distractor.

\begin{algorithm}[b]
\caption{Sequential Greedy Search}\label{alg:greedy}
\begin{algorithmic}[1]
\REQUIRE Instance to be explained $x_{test}$, class of interest $c$, model $f$,
distractor $x_{dist}$, stopping condition $\tau$
\REQUIRE $f_c(x_{dist}) \geq \tau$
\ENSURE $f_c(x') \geq \tau$
\STATE $AF \leftarrow 0_{m\times m}$ \COMMENT{$AF$ is final $A$}
\LOOP
\STATE $x' \leftarrow (I_m - AF)x_{test} + AFx_{dist}$
\STATE $p \leftarrow f_c(x')$
\LINEIFRETURN{$p \geq \tau$}{$AF$}
\FOR{$i \in [0, m]$}
\STATE $A \leftarrow AF$
\STATE $A_{i,i} = 1$
\STATE $x' \leftarrow (I_m - A)x_{test} + Ax_{dist}$
\STATE improvement $\leftarrow f_c(x') - p$
\ENDFOR
\STATE Set $AF_{i,i} = 1$ for $i$ that gives best improvement
\ENDLOOP
\end{algorithmic}
\end{algorithm}

\subsection{Choosing Distractors}

After finding the best $A$ for each $x_{dist}$, we
return the best overall solution as the explanation. As we seek to find the
minimum number of substitutions, it is intuitive to start with distractors
that are as similar to the test sample as possible. Hence, we use the $n$
nearest neighbors of $x_{test}$ in the
training set that are correctly classified as the class of interest for the
distractor. Especially for data sets where samples of the same class can have
different characteristics, e.g., runs of different HPC applications undergoing the
same type of performance anomaly, choosing a distractor similar to $x_{test}$
would intuitively yield minimal and meaningful explanations.

To quickly query for nearest neighbors, we keep all correctly classified
training set instances in a different {\em KD-Tree} per class. The number of
distractors to try out is given by the user as an input to our algorithm, depending
on the running time that is acceptable for the user. If the number of training
instances is large, users may choose to either randomly sample or use
algorithms like $k$-means to reduce the number of training instances before
constructing the KD-tree. The distance measure we use is Euclidean distance, and
we use the KD-tree implementation in scikit-learn~\cite{scikit-learn}.

\subsection{Sequential Greedy Approach}\label{sec:greedy}

The greedy algorithm for solving the hitting set problem is shown to have an
approximation factor of $log_2|U|$, where $U$ is the union of all the
sets~\cite{Chandrasekan:2011}. Thus, one algorithm
we use to generate explanations is the {\em Sequential Greedy Approach}, shown
in Algorithm~\ref{alg:greedy}. We replace each feature
in $x_{test}$ by the corresponding feature from $x_{dist}$. In each iteration, we
choose the feature that leads to the highest increase in the prediction
probability. After we replace a feature in $x_{test}$, we continue the greedy
search with the remaining feature set until the prediction probability exceeds
$\tau$, which is the predefined threshold for the probability values.

\begin{algorithm}[b]
\caption{Random Restart Hill Climbing}\label{alg:climbing}
\begin{algorithmic}[1]
\REQUIRE Instance to be explained $x_{test}$, class of interest $c$, model $f$,
distractor $x_{dist}$, loss function $L(f, c, A, x')$, max attempts, max iters
\FOR{$i \in [0, num_{restarts}]$}

\STATE Randomly initialize $A$; attempts $\leftarrow 0$; iters $\leftarrow 0$
\STATE $x' \leftarrow (I_m - A)x_{test} + Ax_{dist}$
\STATE $l \leftarrow L(f, c, A, x')$

\WHILE{attempts $\leq$ max attempts \textbf{and} iters $\leq$ max iters}
\STATE iters$++$
\STATE $A_{tmp} \leftarrow$ RandomNeighbor($A$)
\STATE $x' \leftarrow (I_m - A_{tmp})x_{test} + A_{tmp}x_{dist}$
\IF{$L(f, c, A_{tmp}, x') \leq l$}
\STATE attempts $\leftarrow 0;$ $A \leftarrow A_{tmp};$ $l \leftarrow L(f, c, A,
x')$
\ELSE
\STATE attempts$++$
\ENDIF
\ENDWHILE
\ENDFOR

\end{algorithmic}
\end{algorithm}

\subsection{Random-Restart Hill Climbing}

Although the greedy method is able to find a minimal set of explanations, searching for the best
explanation by substituting the metrics one by one can become slow for data sets
with many metrics. For a faster algorithm, we apply derivative-free optimization algorithms
to minimize the loss $L$ (in Eqn.~\eqref{eqn:modloss}).

For optimizing running time, we use a hill-climbing optimization method, which
attempts to iteratively improve the current state by choosing the best successor
state under the evaluation function. This method does not construct a search
tree to search for available
solutions and instead it only looks at the current state and possible states in the
near future~\cite{Russell:2009}. It is easy for hill-climbing to settle in
local minima, and one easy modification is {\em random restarting}, which
leads to a so-called {\em Random Restart Hill-Climbing}, shown in
Algorithm~\ref{alg:climbing}.

This algorithm starts with a random initialization point for $A$, and evaluates
$L$ for random neighbors of $A$ until it finds a better neighbor. If a better
neighbor is found, the search continues from the new $A$. In our implementation, we
use the Python package {\tt mlrose}~\cite{Hayes:2019}.

In some cases, hill climbing does not find a viable set $A$ that increases the target
probability. We check for this possible scenario by pruning the output, i.e.,
removing metrics that do not impact target probability. Then, if no metrics
are left, we use greedy search~(\S~\ref{sec:greedy}) to find a viable solution.

\subsection{How to Measure Good Explanations}\label{sec:metrics}
The goal of a local explanation is to provide more information to human
operators to let them understand a particular decision made by an ML
model, learn more about the model, and hypothesize about the future
decisions the model may make. However, in analogy to the Japanese movie Rashomon,
where characters provide vastly different tellings of the same incident, the same
classification can have many possible explanations~\cite{Leo:2001}. Thus, it is necessary to
choose the best one among possible explanations.


There is no consensus on metrics for comparing explainability methods in
academia~\cite{Lipton:2017,Schmidt:2019}. In this work, we aim to provide several
tenets of good explanations with our explainability method.

\textbf{Faithfulness to the original model:}
An explanation is faithful to the classifier if it reflects the actual reasoning
process of the model. It is a first-order requirement of any explainability
method to accurately reflect the decision process of the classifier and not
mislead users~\cite{Ribeiro:2016}. However, most of the time it is challenging
to understand the actual reasoning of complicated ML models. To test the
faithfulness of our method, we explain a simple model with a known reasoning
process and report the precision and recall of our explanations.

\textbf{Comprehensibility by human operators:}
Understanding an explanation should not require specialized knowledge about
ML. According to a survey by Miller~\cite{Miller:2018}, papers
from philosophy, cognitive psychology/science and social psychology should be
studied by explainable artificial intelligence researchers. In the same survey,
it is stated that humans prefer only 1 or 2 causes instead of an explanation
that covers the actual and full list of causes. This is especially important for multivariate
data sets such as
HPC time series data, since each time series represents a different metric and
each metric typically requires research to understand the meaning. Thus, to
evaluate comprehensibility, we compare the {\em number of time series} that are
returned in explanations by different explainability methods.

\textbf{Robustness to changes in the sample:}
A good explanation would not only explain the given sample, but provide similar
explanations for similar samples~\cite{Alvarez:2018,Alvarez:2018a}, painting a
clearer picture in the minds of human operators. Of course, if similar samples
cause drastic changes in model behavior, the explanations should also reflect this.
A measure that have been used to measure robustness is the {\em local Lipschitz
constant} $\mathcal{L}$~\cite{Alvarez:2018},
which is defined as follows for a given $x_{test}$ instance:
\begin{equation}\label{eqn:lipschitz}
\mathcal{L}(x_{test}) = \underset{ x_j \in \mathcal{N}_{k}(x_{test})
  }{\max} \frac{\lVert \xi(x_{test}) - \xi(x_j) \rVert_{2}
  }{ \lVert x_{test} - x_j \rVert_2},
\end{equation}
where $\xi(x)$ is the explanation for instance $x$, and
$\mathcal{N}_{k}(x)$ is the $k$-nearest neighbors of $x_{test}$ in the training set.
We use nearest neighbors, instead of randomly generated samples,
because it is challenging to generate realistic random time series.
The maximum constant is chosen because the explanations should be robust against
the worst-case. Intuitively, the Lipschitz constant measures the ratio of change
of explanations to changes in the samples. We change explanations to $1\times m$
binary matrices (1 if metric is in explanation, 0 otherwise) to be able to subtract them.

\textbf{Generalizability of explanations:}
Each explanation should be generalizable to similar samples;
otherwise, human operators using the explanations would not be able to gain an
intuitive understanding of the model. Furthermore, for misclassifications, it is
more useful for the explanations to uncover classes of misclassifications
instead of a single mishap.

We measure generalizability by applying an explanation's substitutions to other
samples. If the same metric substitutions from the same distractor can flip the
prediction of other samples, that means the explanation is generalizable.


\section{Experimental Setup}\label{sec:setup}
This section describes the time series data sets and ML frameworks we
use to evaluate our explainability method as well as the baseline
explainability methods we implement for comparisons.

\subsection{Data Sets}
We use four high-dimensional multivariate time series data sets:
three HPC system telemetry data sets and a motion classification data set.

For all data sets, we normalize the data such that each time series is between 0
and 1 across the training set. We use the same normalization parameters for the
test set. We use normalized data to train classifiers, and provide normalized
data to the explainability methods. However, the real values of metrics are
meaningful to users (e.g., CPU utilization \%), so we provide un-normalized data
in the explanations given to users and our figures in the paper.

\textbf{HPAS data set:} We use the HPC performance anomaly suite (HPAS)~\cite{Ates:2019}
to generate synthetic performance anomalies on HPC applications and
collect time series data using LDMS~\cite{Agelastos:2014}. We run our
experiments on Voltrino at Sandia National Laboratories, a 24-node Cray XC30m
supercomputer with 2 Intel Xeon E5-2698 v3 processors and 125 GB of memory per
node~\cite{voltrino}.
We run Cloverleaf, CoMD, miniAMR, miniGhost, and miniMD from the Mantevo
Benchmark Suite~\cite{Heroux:2009}, proxy applications Kripke~\cite{kunen:2015}
and SW4lite~\cite{SW4lite}, and MILC which represents part of the codes written
by the MIMD Lattice Computation collaboration~\cite{MILC}. We run each
application on 4 nodes, with and without anomalies. We use the cpuoccupy,
memorybandwidth, cachecopy, memleak, memeater and netoccupy anomalies from HPAS.

Each sample has 839 time series, from the \texttt{/proc} filesystem and
Cray network counters. We take a total of 617 samples for our data set, and we 
divide this into 350 training samples and 267 test samples. One sample
corresponds to the data collected from a single node of an application run.
After this division, we extract 45 second time windows with 30 second overlaps
from each sample.

\textbf{Cori data set:} We collect this data set from Cori~\cite{cori} to test our explainability
method with data from large-scale systems and real applications. The goal of this data set
is to use monitoring data to classify applications. Cori is a Cray XC40
supercomputer with 12,076 nodes. We run our applications in
compute nodes with 2 16-core Intel Xeon E5-2698 v3
processors and 128 GB of memory. We run 6 applications on 64 nodes for 15-30 minutes.
The applications are 3 real applications,
LAMMPS~\cite{LAMMPS}, a classical molecular dynamics code with a focus on
materials modeling, QMCPACK~\cite{QMCPACK}, an open-source continuum quantum
Monte Carlo simulation code, HACC~\cite{HACC}, an open-source code uses N-body
techniques to simulate the evolution of the universe; 2 proxy applications, NEKBone
and miniAMR from ECP Proxy Apps Suite~\cite{Proxy}; and HPCG~\cite{HPCG} benchmark
which is used to rank the TOP500 computing systems.

We collect a total of 9216 samples and we divide this into 7373 training
and 1843 test samples.
Each sample represents the data collected from a single node of an application
run and has 819 time series collected using LDMS from the \texttt{/proc}
filesystem and PAPI~\cite{papi_counters} counters.

\textbf{Taxonomist data set:} This data set, released by Ates et
al.~\cite{taxonomist_artifact}, was collected from Voltrino, a Cray
XC30m supercomputer, using LDMS. The data set contains runs of 11 different
applications with various input sets and configurations, and the goal is again to classify
the different applications.

We use all of the data, which has 4728 samples. We divide it into 3776 training
samples and 952 test samples. Each sample has 563 time series. Each sample
represents the data collected from a single node of an application run.

\textbf{NATOPS data set:} This data set is from the motion classification
domain, released by Ghouaiel et al.~\cite{natops}\footnote{We use the version
found in UCR time series classification
repository~\cite{timeseriesclassification}}. We chose this data set because of
the relatively high number of time series per sample, compared to other time
series data sets commonly used in the ML domain.

The NATOPS data contains a total of 24 time series representing the X, Y and Z
coordinates of the left and right hand, wrist, thumb and elbows, as captured by
a Kinect 2 sensor. The human whose motions are recorded repeats a set of 6 Naval Air
Training and Operating Procedures Standardization (NATOPS) motions meaning ``I
have command,'' ``All clear,'' ``Not clear,'' ``Spread wings,'' ``Fold wings,''
and ``Lock wings.'' We keep the original training and test set of 180 samples
each, with 50 second time windows.

\subsection{Machine Learning Techniques}
We evaluate our explainability techniques by explaining 3 different ML
pipelines that represent different anlaytics frameworks proposed by researchers.

\textbf{Feature Extraction + Random Forest:}
This technique represents a commonly used pipeline to classify time series
data for failure prediction, diagnose performance variability, or classify
applications~\cite{Tuncer:2017,Tuncer:2019,Ates:2018,Klinkenberg:2017,Nie:2018}.
For example, Tuncer et al.~\cite{Tuncer:2017} diagnose performance anomalies at
runtime by collecting time series data with different types of anomalies, and
train a random forest to classify the type, or absence, of anomalies using
statistical features extracted from time series.

This method is not explainable because the random forests produced can be very
complex. For example, the random forest we trained with the HPAS data set had
100 trees and over 50k nodes in total. Operators that try to understand a prediction
without explainability methods would have to inspect the decision path through
each decision tree to understand the mechanics of the decision, and
understanding high-level characteristics such as ``how can this misclassification
be fixed?'' is near-impossible without explainability techniques.

We extract 11 statistical features including the minimum, maximum, mean, standard
deviation, skew, kurtosis, 5\textsuperscript{th}, 25\textsuperscript{th},
50\textsuperscript{th}, 75\textsuperscript{th} and 95\textsuperscript{th}
percentiles from each of the time series. Then, we train scikit-learn's random
forest classifier based on these features~\cite{scikit-learn}.

\textbf{Autoencoder:}
Borghesi et al. have proposed an
autoencoder architecture for anomaly detection using HPC time series
data~\cite{Borghesi:2019,Borghesi:2019a}. The autoencoder is trained using only
``healthy'' telemetry
data, and it learns a compressed representation of this data. At runtime, data
is reconstructed using the autoencoder and the mean error is measured. A high
error means the new data deviates from the learned ``healthy'' data; thus, it can
be classified as anomalous. We implement the architecture described by Borghesi
et al. and use it for our evaluation. In order to convert the mean error, which
is a positive real number, to class probabilities between 1 and 0, we subtract
the chosen threshold from the error and use the sigmoid function. This
autoencoder model is a deep neural network, and deep neural networks are known
to be one of the least explainable ML methods~\cite{Gunning:2017}.

\textbf{Feature Extraction + Logistic Regression:}
The logistic regression classifier is inherently interpretable, so we use this
pipeline for sanity checks of our explanations in experiments where we need a
ground truth for explanations. For input feature vector $x$ the logistic
regression model we use calculates the output $y$ using the formula:
\begin{equation*}
y = S(w \cdot x),
\end{equation*}
where $S(z) = \frac{1}{1 + e^{-z}}$ is the sigmoid function.
Thus, the classifier only learns the weight vector $w$ during
training\footnote{Other formulations of logistic regression include a $b$ term
  such that $y = S(w \cdot x + b)$, but we omit this for better
  interpretability.}.
Furthermore, it is possible to deduce that any feature $x_i$ for which the
corresponding weight $w_i$ is zero has no effect on the classification.
Similarly, features can be sorted based on their impact on the classifier decision
using $|w_i|$. We use the same features as the random forest pipeline.

\subsection{Baseline Methods}
We compare our explainability method with popular explainability methods,
LIME~\cite{Ribeiro:2016}, SHAP~\cite{Lundberg:2017}, as well as Random, which
picks a random subset of the metrics as the explanation.

\subsubsection{LIME}
LIME stands for local interpretable model-agnostic explanations~\cite{Ribeiro:2016}.
LIME operates by fitting an interpretable linear model to the classifiers predictions
of random data samples. The samples are weighted based on
their distance to the test sample, which makes the explanations local. When
generating samples, LIME generates samples within the range observed in the
training set. In our
evaluation, we use the open-source LIME implementation~\cite{lime_package}.

LIME does not directly apply to time series as it operates by sampling the classifier
using randomly generated data. Randomly generating complex multivariate time series data such as HPC telemetry data while
still obeying the possible constraints in the data as well as maintaining
representative behavior over time is a challenging open problem.
In our evaluation, we apply LIME to frameworks that perform feature extraction,
and LIME interprets the classifier that takes features as input.

Another challenge with LIME is that it requires the number of features
in the explanation as a user input. Generally, it is hard for users to know
how many features in an explanation are adequate. In our experiments, we
use the number of metrics in our method's explanation as LIME's input.

\begin{figure}[tb]
  \centering
  \includegraphics[width=\columnwidth]{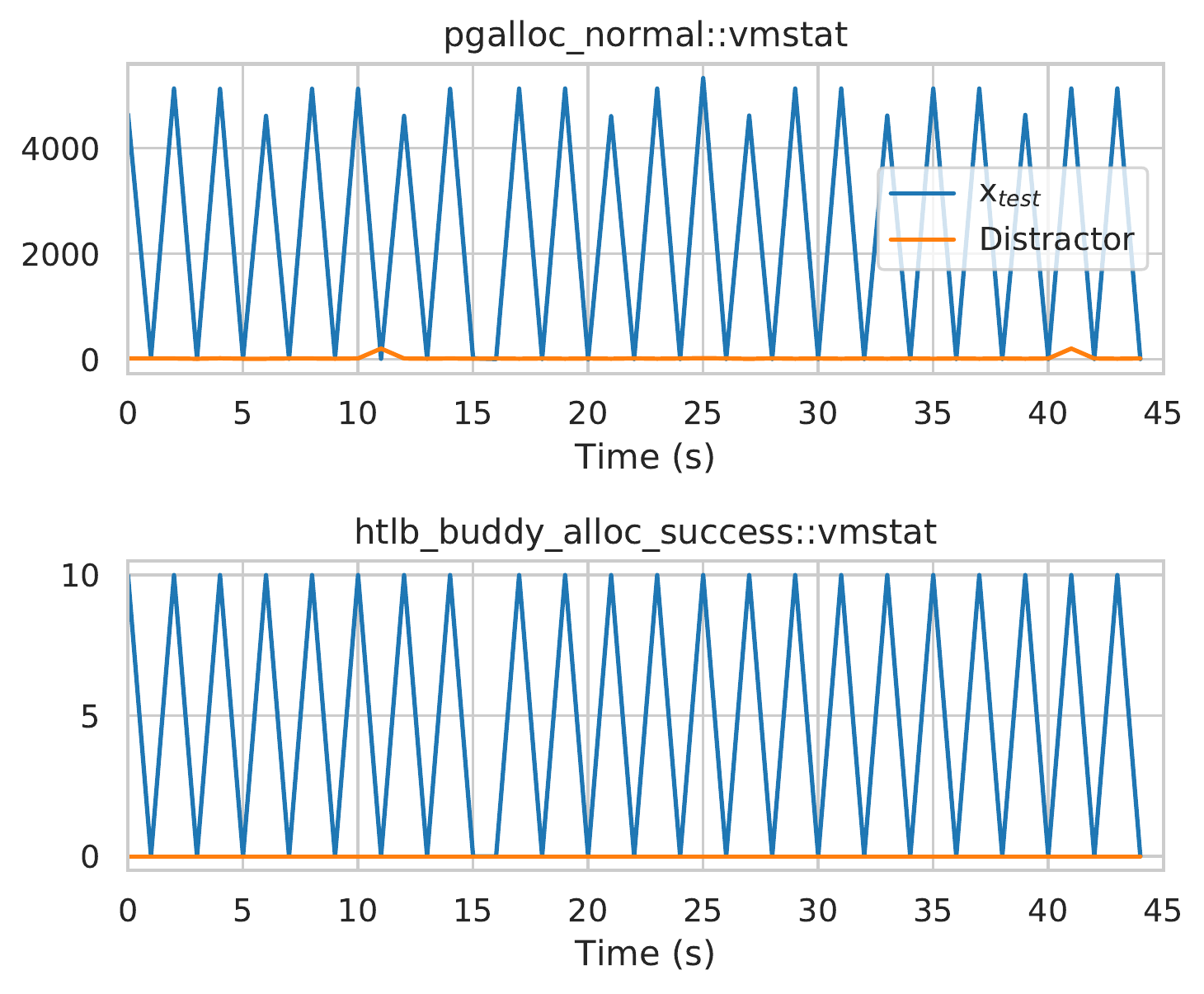}
  \vspace{-0.3in}
  \caption{The explanation of our method for correctly  classified  time  window
    with  the  ``memleak" label. Our method provides two metrics as an
    explanation to change classification label from ``memleak'' to ``healthy.''
The metrics are shown in the $y$-axes and the metric names are above the
    plots. The first metric indicates that the classifier is looking for
    repeated memory allocations in runs with memory leak. The second metric
    indicates that the classifier is looking for repeated successful huge page
    allocations in runs with memory leak. The second metric may indicate that
    the data set we use is biased and does not include runs with high memory
    fragmentation.
  }\label{fig:our_method}
  \vspace{-0.2in}
\end{figure}

\subsubsection{SHAP}
The Shapley additive explanations (SHAP), presented by Lundberg and
Lee~\cite{Lundberg:2017}, propose 3 desirable characteristics of explanations:
local accuracy, missingness, and consistency. They define additive SHAP
values, i.e., the importance values can be summed to arrive at the classification.
SHAP operates by calculating feature importance values by using model parameters;
however, since we do not have access to model parameters, we use KernelSHAP
which estimates SHAP values without using model weights.

We use the open-source KernelSHAP implementation~\cite{shap_package}, which we refer
to as SHAP in the remainder of the paper. SHAP also suffers from one of the
limitations of LIME; it is not directly applicable to time series, so we apply SHAP
to frameworks that perform feature extraction. SHAP
does not require the number of features in the explanation as an input.


\section{Evaluation}\label{sec:eval}
In this section, we evaluate our explainability method and compare it with other
explainability methods based on qualitative comparisons and the
metrics described in Sec.~\ref{sec:metrics}. We aim to answer several questions:
(1) Are the explanations minimal? (2) Are the explanations
faithful to the original classifier? (3) Are the explanations robust, or do we
get different explanations based on small perturbations of the input? (4) Are
the explanations generalizable to different samples? (5) Are the explanations
useful in understanding the classifier?

\begin{figure}[tb]
  \centering
  \includegraphics[width=\columnwidth]{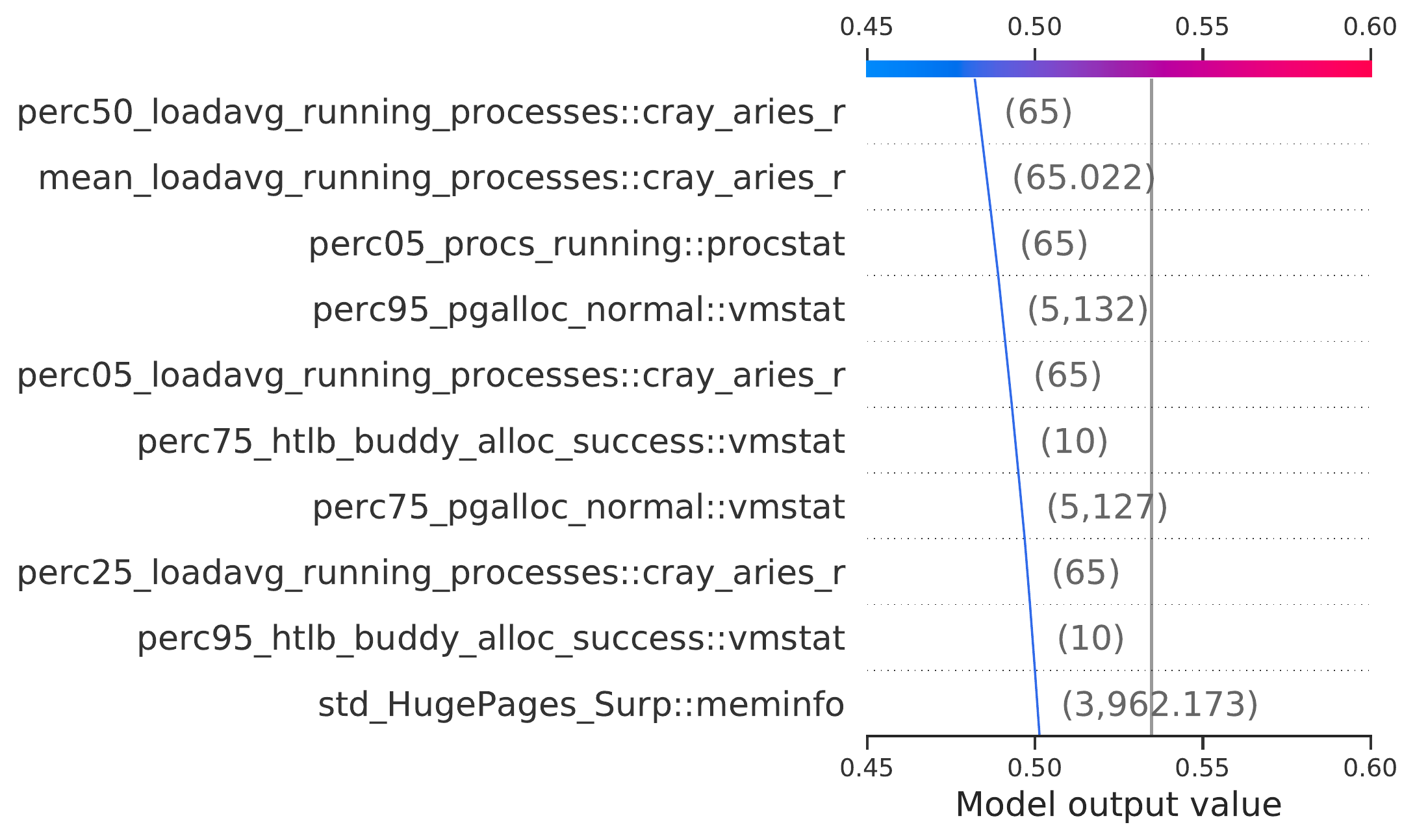}
  \vspace{-0.2in}
  \caption{The explanation of SHAP for a correctly  classified  time  window
    with  the  ``memleak'' label. SHAP provides 187 features (only 10 shown), where non-zero SHAP
    values are expected to explain the characteristics of the ``memleak" anomaly.
    It is very challenging to understand how many
    features are sufficient for an explanation, whether these features are
    relevant to ``memleak'' or other anomalies, or how to interpret the feature
    values in the explanation.
  }\label{fig:shap}
  \vspace{-0.1in}
\end{figure}

\subsection{Qualitative Evaluation}\label{sec:eval:qual}
Our first-order evaluation is to use our explanation technique and the baselines
to explain a realistic classifier. Similar to the framework proposed by Tuncer et
al.~\cite{Tuncer:2017,Tuncer:2019}, we use the random forest classifier with
feature extraction and the HPAS dataset, which includes different types of performance
anomalies. We choose ``memleak'' anomaly from HPAS data set, which makes increasing
memory allocations without freeing to mimic memory leakage. Our goal is to better
understand the classifier's understanding of the ``memleak'' anomaly. After training the random
forest pipeline, we choose a correctly
classified time window with the memleak label as $x_{test}$, and the ``healthy''
class as the class of interest. We run our method,
LIME, and SHAP with the same $x_{test}$ and compare the results.

Our explanation contains two time series, and is shown in
Fig.~\ref{fig:our_method}. The first metric to be substituted is {\tt
pgalloc\_normal} from {\tt /proc/vmstat}, which is a counter that
represents the number of page allocations. Because of our preprocessing, the
plot shows the number of page allocations per second. It is immediately clear
that the nodes with memory leaks perform many memory allocations and act in a
periodic manner.

The second metric in Fig.~\ref{fig:our_method} is {\tt
htlb\_bud\-dy\-\_\-al\-loc\-\_\-suc\-cesses}, which also belongs to the same time window.
This metric shows the number of successful huge page allocations.
Memory leaks do not need to cause huge page allocations, since memory leaks in
a system with fragmented memory might cause failed
huge page allocations. This indicates that our training set is biased
towards systems with less fragmented memory, most probably because our
benchmarks are all short-lived. 

The SHAP explanation, in Fig.~\ref{fig:shap}, contains 187 features with very
similar SHAP values. Even though
we can sort the features by importance, it is difficult to decide how many
features are sufficient for a good explanation. Also, SHAP provides a single
explanation for one sample, regardless of which class we are interested in, so
the most
important features are features that are used to differentiate this run from
other CPU-based anomalies, which may not be relevant if our goal is to understand
memory leak characteristics. Finally, it is left to the user to interpret the
values of different features, e.g., the 75\textsuperscript{th} percentile of
{\tt pgalloc\_normal} was 5,127; however, this does not
inform the user of normal values for this metric, or whether it was too high or
too low.

\begin{figure}[tb]
  \centering
  \includegraphics[width=\columnwidth]{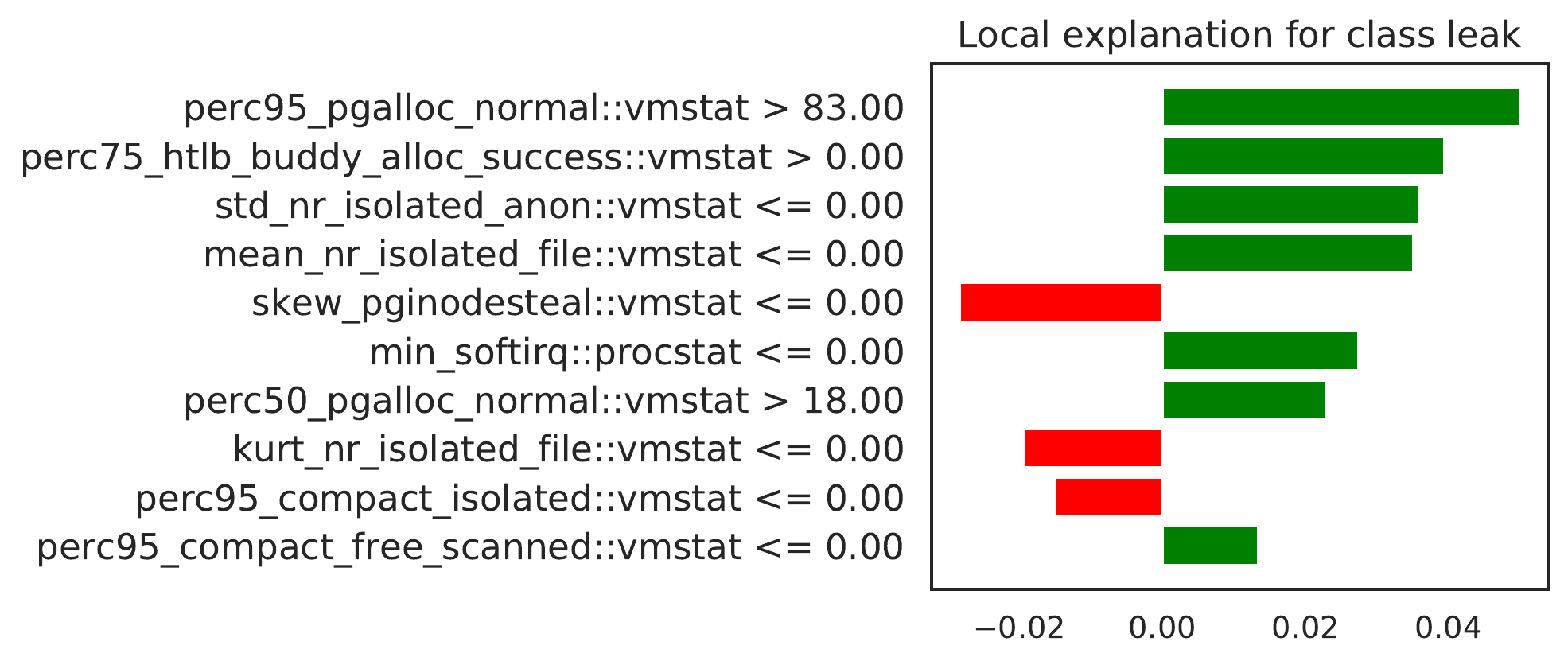}
  \vspace{-0.3in}
  \caption{The explanation of LIME for correctly classified time window with the
    ``memleak'' label. LIME provides features that positively (green) and
    negatively (red) affect the decision. Although the first two features are
    derived from the metrics in our explanation, it is not straightforward to
    interpret the values of the features, especially the negative ones. The
    number of features is an input from the user. }\label{fig:lime}
  \vspace{-0.1in}
\end{figure}

The LIME explanation is shown in Fig.~\ref{fig:lime}. Green values indicate that
the features were used in favor of memory leak, and red values were opposing
memory leak. We keep the number of features in the explanation at the default value
of 10. The first two features are derived from the metrics in our explanation, and
a threshold is given for the features, e.g., the 95\textsuperscript{th}
percentile of page allocations is over 83, which causes this run to be likely to
be a memory leak. Interpreting features such as percentiles, standard deviation
and thresholds on their values is left to the user.
Furthermore, the effect of the red features is unclear, as it is not stated which
class the sample would be if it is not labeled as leak.

It is important to note that both LIME and SHAP use randomly generated data for
the explanations. In doing so, these methods assume that all of the features are
independent variables; however, many features are in fact dependent, e.g.,
features generated from the same metric. Without knowledge of this, these random
data generation methods may test the classifier with synthetic runs that are
impossible to get in practice, e.g., synthetic runs where the
75\textsuperscript{th} percentile of the one metric is lower than the
50\textsuperscript{th} percentile of the same metric. Our method does not
generate synthetic data, and uses the whole time series instead of just the
features, so it is not affected by this.

\subsection{Comprehensibility}
We measure comprehensibility using the number of metrics in the explanation. Our
method returns 2 time series for the qualitative evaluation example in
Fig.~\ref{fig:our_method}, and in most cases the number of time series in our
explanations is below 3; however, for some challenging cases it can reach up to
10. SHAP returns 187 features in Fig.~\ref{fig:shap}, and SHAP explanations
typically have hundreds of features for HPC time series data. LIME requires the
number of features as an input; however, does not provide any guidelines on how
to decide this value.

\subsection {Faithfulness Experiments}

We test whether the explainability methods actually reflect the decision process
of the models, i.e., whether they are faithful to the model. For every data set, we train a
logistic regression model with $L1$ regularization. We change the $L1$
regularization parameter until less than 10 features are used by the classifier.
The resulting classifier uses 5 metrics for HPAS and Cori data, 9 for NATOPS and 8 for
Taxonomist. Because we know the used features, we can rank the explanations
based on precision and recall.

\begin{itemize}
\item \textbf{Recall:} How many of the metrics used by the classifier are in
  the explanation?
\item \textbf{Precision:} How many of the metrics in the explanation are used
  by the classifier?
\end{itemize}

\begin{figure}[t]
  \centering
  \includegraphics[width=\columnwidth]{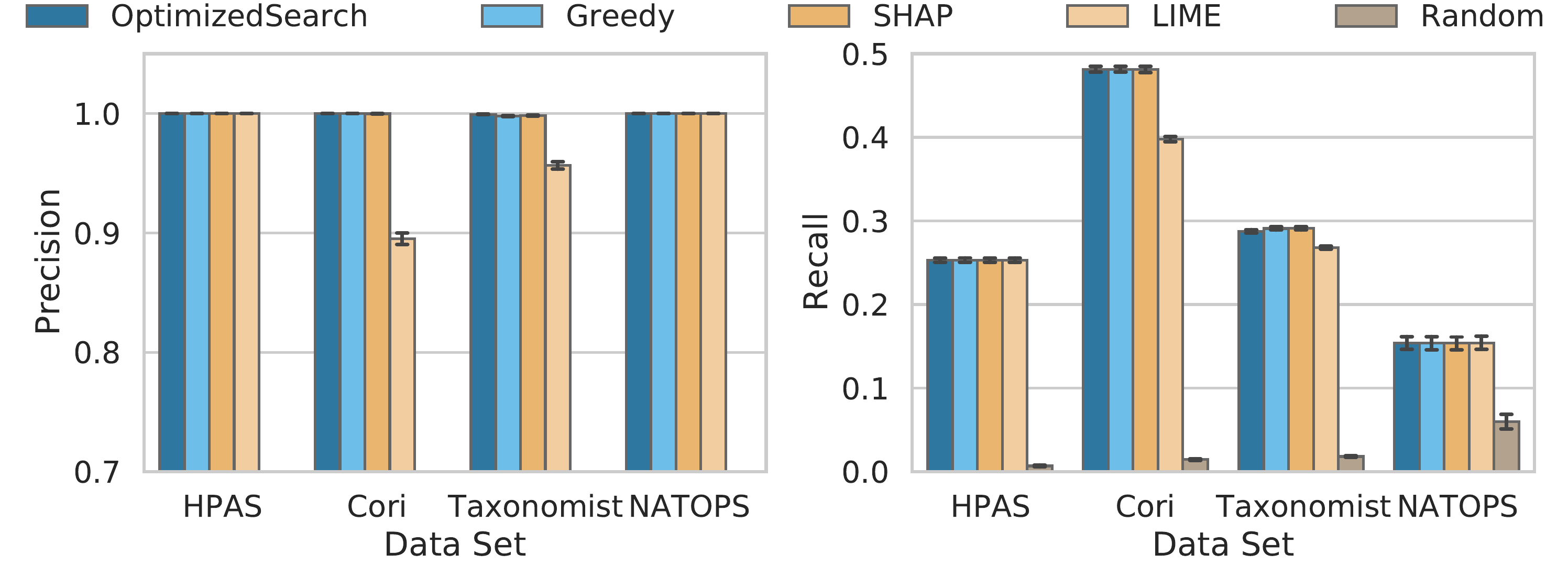}
  \vspace{-0.3in}
  \caption{Precision and recall of the explanations for a classifier with known
    feature importances. Our proposed explainability method (OptimizedSearch and
    Greedy), and SHAP have perfect precision. LIME has lower precision for Cori and
    Taxonomist data sets, which
    indicates that although some features have no impact on the classifier decision,
    they are included in the LIME explanations. The low recall indicates that not
    every feature is used in every local decision.}\label{fig:faithfulness}
  \vspace{-0.1in}
\end{figure}

We acquire explanations for each sample in the test set, and show the average
precision and recall in Fig.~\ref{fig:faithfulness}. To ensure that the
 other explainability methods are not at a disadvantage, we first run the greedy search method
and get the number of metrics in the explanation. Then, we get the same number
of metrics from each method. In this way, as an example, LIME is not adversely affected by
providing 10 features in the explanation even though only 7 are used
by the classifier.

The results show that both our method and SHAP have perfect precision. Recall values of
the explanations are lower than 1 because not every feature in the classifier is
effective for every decision. Notably, LIME has low precision for the
Cori and Taxonomist data sets, which indicates that there may be features in LIME
explanations that are actually not used by the classifier at all. This outcome could be
due to the randomness in the data sampling stage of LIME.

\subsection {Robustness Experiments}

\begin{figure}[t]
  \centering
  \includegraphics[width=\columnwidth]{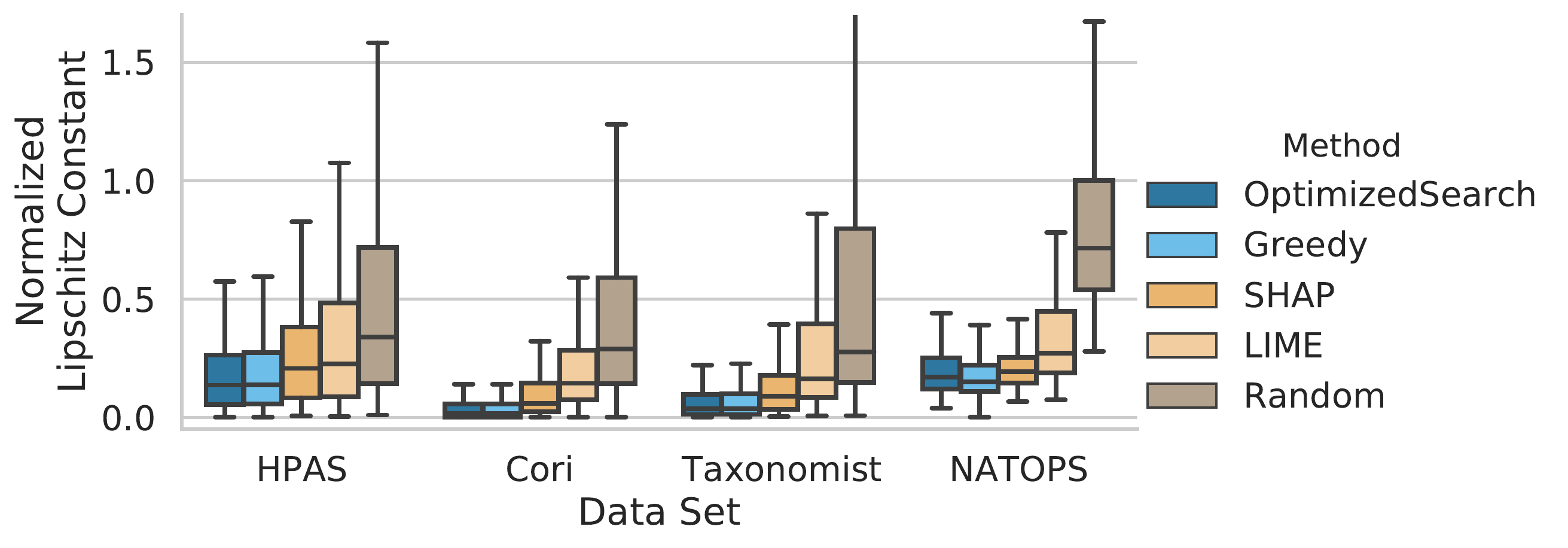}
  \vspace{-0.3in}
  \caption{Robustness of explanations to changes in the test sample.
    Our proposed explainability method (OptimizedSearch and Greedy) is the most
    robust to small changes in the
    input, resulting in more predictable explanations and better user
    experience. Lipschitz constant is normalized to be comparable between
    different data sets, and a lower value indicates better robustness.
    }\label{fig:robustness}
  \vspace{-0.2in}
\end{figure}

\begin{figure}[b]
  \centering
  \vspace{-0.2in}
  \includegraphics[width=0.8\columnwidth]{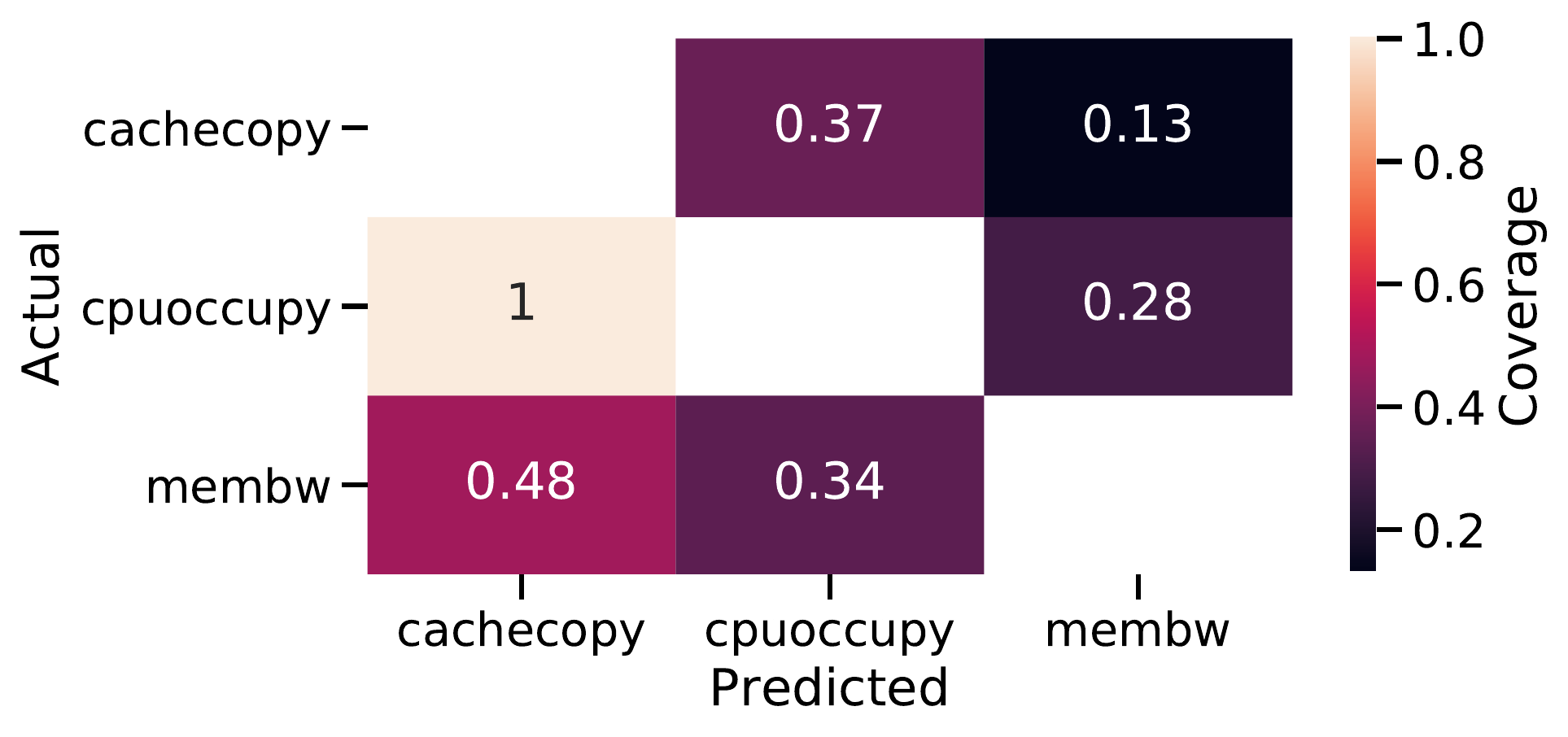}
  \vspace{-0.15in}
  \caption{The ratio of test samples that our explanations are applicable to,
    among samples with the same misclassification characteristics. For the
    cpuoccupy runs that are misclassified as cachecopy, every explanation is
    applicable to every other sample with the same misclassification.
  }\label{fig:generality}
  \vspace{-0.1in}
\end{figure}

For robustness, we calculate the Lipschitz constant~\eqref{eqn:lipschitz}
for each test sample and show average results in
Fig.~\ref{fig:robustness}. According to the results, our method is the most robust
explainability
technique. One reason is that our method does not involve random data
generation for explanations, which reduces the randomness in the explanations.
It is important for explanations to be robust, which ensures that users can
trust the ML models and explanations. For the NATOPS data set, the greedy method
has better robustness compared to optimized search, because the greedy
method inspects every metric before generating an explanation, thus finds the
best metric, while the optimized search can stop after finding a suitable
explanation even if there can be better solutions.

\subsection {Generalizability Experiments}

We test whether our explanations for one $x_{test}$ are generalizable to other
samples. We use the HPAS data set and random forest
classifier with feature extraction. There are 3 classes that are
confused with each other. For each misclassified test instance, we get an
explanation and apply the same metric substitutions using the same distractor to
other test samples with the same (true class, predicted class) pair.

We report
the percentage of misclassifications that the explanation applies to (i.e.,
successfully flips the prediction for) in Fig.~\ref{fig:generality}.
According to our results, on average, explanations for one mispredicted sample are
applicable to over 40\% of similarly mispredicted samples. This shows
that users do not need to manually inspect the explanation for every
misprediction, and instead they can obtain a general idea of the classifiers
error characteristics from a few explanations, which is one of the goals of
explainability.

\begin{figure}[t]
  \centering
  \includegraphics[width=\columnwidth]{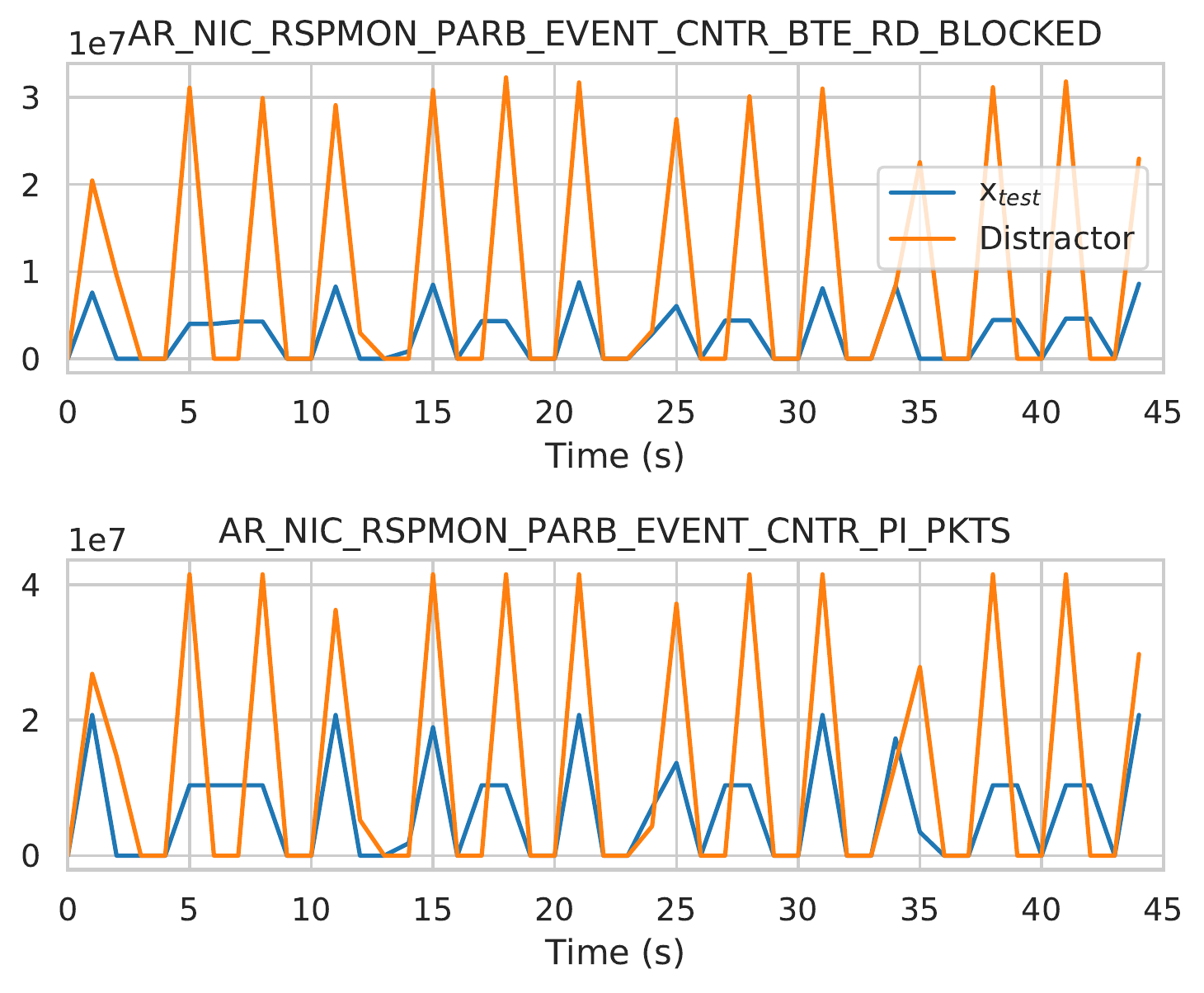}
  \vspace{-0.3in}
  \caption{Our explanation for a ``network'' anomaly misclassified as
    ``healthy.'' The metrics are shown in the $y$-axes and the metric names are
    above the plots. The explanation indicates that the anomaly needs higher
    network traffic to be classified correctly. 4 of the 6 metrics in the
    explanation are omitted because they appear identical to the metrics shown.
  }\label{fig:autoencoder}
  \vspace{-0.2in}
\end{figure}

\subsection{Investigating Misclassifications}
As a demonstration, we debug a misclassified sample using our explainability
method. This is a typical scenario that would be encountered if ML
systems are deployed to production. We train the autoencoder-based
anomaly detection framework~\cite{Borghesi:2019,Borghesi:2019a} using healthy
data from the HPAS data set. Among the
runs with the network anomaly, the run shown in Fig.~\ref{fig:autoencoder} is
misclassified. We explain this misclassification using our explanation method:
we choose the misclassified run as $x_{test}$ and the anomalous class as the
class of interest (the autoencoder has two classes: anomalous
and healthy). We cannot apply the LIME and SHAP baselines
as the autoencoder directly takes time series as input.

\begin{table}[b]
\vspace{-0.2in}
\centering
\caption{Network metric names in explanation. Full names are \newline
``ar\_nic\_(Field 1)\_event\_cntr\_(Field 2)\_(Field 3)''
}\label{tab:metrics}
\begin{tabular}{llll}
\toprule
Metric & Field 1 & Field 2 & Field 3\\
\midrule
1 & RSPMON\_PARB & PI & FLITS\\
2 & RSPMON\_PARB & PI & PKTS\\
3 & RSPMON\_PARB & AMO & FLITS\\
4 & RSPMON\_PARB & AMO & PKTS\\
5 & NETMON\_ORB\ & EQ & FLITS\\
6 & RSPMON\_PARB & BTE\_RD & BLOCKED\\
\bottomrule
\end{tabular}
\end{table}

The explanation includes 6 network metrics from Table~\ref{tab:metrics}.
Metrics 6 and 2 are shown in Fig.~\ref{fig:autoencoder}. Field 3
describes if the metric counts the number of flits or packets. Blocked means a
flit was available at the input but arbitration caused the selection of another
input. Metrics 1 and 2 count traffic being forwarded by the network
interface card (NIC) to the processor, 3 and 4 count the processor memory
read and write traffic resulting from requests received over the network, metric 5
counts traffic injected by the NIC into the network, and metric 6
counts reads of processor memory initiated by the NIC's block transfer engine to
fetch data included in the put requests it is generating~\cite{cray_counters}.

The explanation indicates that the intensity of the network anomaly in this run
needs to be higher, i.e., more network traffic is needed, for this to be
classified as a network anomaly. Furthermore, since 6 metrics all need to be
modified for the prediction to flip, it can be seen that the autoencoder has
learned the parallel behavior of these metrics. For example, if the number of
packets is changed independent of the number of flits, the classifier does not
change its prediction. It is highly unlikely that randomly generated samples
would capture the correlated behavior of these two metrics.


\section{Conclusion and Future Work}\label{sec:eval:explainable}
This paper, for the first time, investigated explainability for ML frameworks
that use multivariate time series data sets, with a focus on the HPC domain.
Multivariate time series data is widely used in many scientific and engineering
domains, and ML-based HPC analysis and management methods that show a lot of
promise to improve HPC system performance, efficiency, and resilience. Being
explainable is an important requirement for any ML framework that seeks
widespread adoption.

We defined the counterfactual time series explainability problem and
presented a heuristic algorithm that can generate feasible explanations. We also
demonstrated the use of our explanation method to explain various frameworks,
and compared with other explainability methods. We identify the following open
problems for future exploration.

\textbf{Minimizing probabilities}
Our problem statement maximizes the prediction probability for the target
class. For binary classification, this is equivalent to minimizing the
probability for the other class, but for multi-class classification, there can be
differences between minimizing and maximizing. We found that in practice both
false positives and true positives
can be explained by maximizing one class' probability, but a similar
explainability method could be designed to minimize class probabilities and get
new explanations.

\textbf{Approximation algorithms}
We have presented a heuristic algorithm; however, we have not provided any
bounds on the optimality of our
algorithm. Finding approximation algorithms with provable bounds is an open problem, as such
algorithms may yield better explanations in a shorter time.

\textbf{Explainable models}
During our experiments, we conducted a preliminary experiment using
CORELS, which is an interpretable model that learns rule lists~\cite{Angelino:2017}.
We trained CORELS using the HPAS data set to
classify the time windows between the 5
anomalies and the healthy class. Regardless of our experimentation with different
configuration options, the model in our experiments took over 24 hours
to train and the resulting rule list classified every time window as
healthy regardless of the input features. Challenges in this direction
include designing inherently explainable models that can give good accuracy for multivariate time
series data.

\textbf{Production systems}
The challenges faced when deploying ML frameworks on HPC systems
can lead to important research questions, and it has been shown that user
studies are critical for evaluating explainability
methods~\cite{Poursabzi:2018}. We hope to work with administrators and deploy
ML frameworks with our explainability method to production HPC systems and
observe how administrators use the frameworks and explanations, and find the
strengths and weaknesses of different approaches.

\appendix[Proof for \NP-Hardness of the counterfactual time series
explainability problem]

In order to prove the \NP-hardness of the counterfactual time series
explainability problem, we first consider a simplified problem of explaining a
binary classifier $f$ by finding a
minimum $A$ such that $f(x') = 1$ for a fixed $x_{dist}$. The proof
we use is similar to the proof by Karlsson et al.~\cite{Karlsson:2018}.

\begin{lemma}
  Given a random forest classifier $f$ that takes $m$ time series of length 1
  as input, a test sample $x_{test}$, and a distractor $x_{dist}$,
  the problem of finding a minimum set of substitutions $A$ such that
  $f(x') = 1$ is \NP-hard.
\end{lemma}

\begin{proof}
  We consider the hitting set problem as follows:
  Given a collection of sets $\Sigma = \{S_1, S_2, \ldots, S_n \subseteq U\}$,
  find the smallest subset $H \subseteq U$ which hits every set in $\Sigma$. The
  hitting set problem is \NP-hard~\cite{Karp:1972}. We enumerate $U$ such that
  each elements maps to a number between 0 and $m$, $U = [0, m]$.

  Assume there is an algorithm to solve our problem that runs in polynomial
  time. We construct a special case of our problem that can be used to solve
  the hitting set problem described above.
  Assume the time series are all of
  length 1 and can have values 0 or 1. Construct a random forest classifier
  $\mathcal{R}(x) = y: \mathbb{R}^{m \times 1} \rightarrow \mathbb{R}$ of $n$
  trees $\mathcal{R} = \{T_1, T_2, \ldots, T_n\}$, $m = |U|$.
  Each tree $T_i$ then classifies the multivariate time series as
  class 1 if any time series of the corresponding subset $S_i \subseteq [0, m]$
  is 1, and classifies as class 0 otherwise. $\mathcal{R}(x)$ thus returns
  the ratio of trees that classify as 1, or the ratio of sets that are covered.
  For $x_{test}$, we use a multivariate time series of all 0s (which will
  be classified as class 0), and as the distractor $x_{dist}$ we use all 1s,
  classified as class 1.

  Our algorithm finds a minimum set of substitutions $A$ such that
  $\mathcal{R}(x') = 1$. We can transform $A$ to the solution of the hitting
  set problem $H$ by adding the $j$\textsuperscript{th} element of $U$ to $H$ if
  $A_{jj} = 1$. Thus, $H \subseteq U$ has minimum size and hits each subset $S_j$,
  i.e., $H \cap S_j \neq \emptyset$ for all $j \in [0, n]$. Thus, we can
  use our algorithm to solve the hitting set problem. Since the hitting set
  problem in \NP-hard, the existence of such a polynomial-time algorithm is unlikely.
\end{proof}

\begin{theorem}
  Given classifier $f$, class of interest $c$, test sample $x_{test}$, and the
  training set $X$ for the classifier, the {\em counterfactual time series
    explainability problem} of finding $x_{dist}$ and $A$ that maximize
  $f_c(x')$ is \NP-hard.
\end{theorem}

\begin{proof}
  Lemma 1 shows the \NP-hardness for a special case of our problem with random
  forest classifiers, binary time series, binary classification and without the
  added problem of choosing a distractor $x_{dist}$. Based on this, the general
  case with real-valued time series and more complicated classifiers is also
  \NP-hard.
\end{proof}


\begin{small}
\section*{Acknowledgment}
This work has been partially funded by Sandia National Laboratories. Sandia National Labs is a multimission laboratory managed and operated
by National Technology and Engineering Solutions of Sandia, LLC., a wholly owned
subsidiary of Honeywell International, Inc., for the U.S. Department of Energy's
National Nuclear Security Administration under Contract DE-NA0003525.
\vspace{-0.1in}
\end{small}

\bibliographystyle{IEEEtran}
\bibliography{main.bib}

\end{document}